\documentclass{article}

\usepackage{microtype}
\usepackage{graphicx}
\usepackage{subfigure}
\usepackage{booktabs} % for professional tables
\usepackage{enumitem}
\usepackage{soul, color, xcolor}
\usepackage{amsthm}

\usepackage[english]{babel}
\newtheorem{theorem}{Theorem}
\newtheorem{lemma}[theorem]{Lemma}

\usepackage{multirow}
\usepackage{booktabs}
\usepackage[skip=0pt,font=small]{caption}
\usepackage{subfigure}
\usepackage[toc,page]{appendix}
\usepackage{natbib}
\usepackage{xspace}
\usepackage{listings}

\usepackage[bookmarks=false]{hyperref}

\newcommand{\name}{ED-Batch\xspace}

\usepackage{algorithm}
\usepackage[noend]{algpseudocode}

\usepackage[arxiv]{icml2020}
\icmltitlerunning{ED-Batch: Efficient Automatic Batching of Dynamic Neural Networks via Learned Finite State Machines}

\begin{document}
\setlength{\textfloatsep}{3pt}
\twocolumn[
\icmltitle{ED-Batch: Efficient Automatic Batching of Dynamic Neural Networks\\ via Learned Finite State Machines}



\icmlsetsymbol{equal}{*}

\begin{icmlauthorlist}
\icmlauthor{Siyuan Chen}{pku}
\icmlauthor{Pratik Fegade}{cmu}
\icmlauthor{Tianqi Chen}{cmu,octoml}
\icmlauthor{Phillip B. Gibbons}{cmu}
\icmlauthor{Todd C. Mowry}{cmu}
\end{icmlauthorlist}

\icmlaffiliation{pku}{Peking University}
\icmlaffiliation{cmu}{Carnegie Mellon University}
\icmlaffiliation{octoml}{OctoML}

\icmlcorrespondingauthor{Siyuan Chen}{chensiyuan@pku.edu.cn}

\icmlkeywords{Machine Learning, ICML}

\vskip 0.3in
]



\printAffiliationsAndNotice{}  % leave blank if no need to mention equal contribution
\begin{abstract}
    Batching has a fundamental influence on the efficiency of deep neural network (DNN) execution. However, for dynamic DNNs, efficient batching is particularly challenging as the dataflow graph varies per input instance. As a result, state-of-the-art frameworks use heuristics that result in suboptimal batching decisions. Further, batching puts strict restrictions on memory adjacency and can lead to high data movement costs. In this paper, we provide an approach for batching dynamic DNNs based on finite state machines, which enables the automatic discovery of batching policies specialized for each DNN via reinforcement learning. Moreover, we find that memory planning that is aware of the batching policy can save significant data movement overheads, which is automated by a PQ tree-based algorithm we introduce. Experimental results show that our framework speeds up state-of-the-art frameworks by on average 1.15x, 1.39x, and 2.45x for chain-based, tree-based, and lattice-based DNNs across CPU and GPU.
    \vspace{-3ex}
\end{abstract}

% Hi Siyuan, are you there?

\section{Introduction}

Batching accelerates the training and inference for deep neural networks (DNN) because it (1) launches fewer kernels resulting in lower kernel launch and scheduling overhead on the CPU, and (2) better utilizes the hardware by exploiting more parallelism. For static DNNs, i.e. DNNs whose dataflow graphs (a.k.a., computation graphs) are identical across every input instance, batched execution is trivial as one can batch corresponding operations for each input together. However, DNNs used to model structured data such as trees~\cite{treenn}, grids~\cite{gatedRNN}, and lattices~\cite{LatticeLSTM} in applications like natural language processing and speech recognition, exhibit dynamism in the network structure. In other words, the dataflow graph for these DNNs varies for each input instance. As a result, batching is a non-trivial problem for these DNNs.

Due to the presence of dynamism, batching for dynamic DNNs cannot be done during compilation. As a result, past work on the efficient execution of dynamic DNNs focused on two directions: (1) to enable operation-level batching at runtime~\cite{tffold, dynet,jit-batching}, i.e. \textit{dynamic batching}, and (2) to extract static parts~\cite{cavs}, i.e.~the static subgraphs (e.g. the LSTM cell), from the dataflow graph and optimize them during compilation~\cite{cortex,acrobat}. Because of strict runtime constraints, the former approach relies on simple heuristics to guide batching, leading to suboptimal performance. In the latter approach, techniques dedicated to certain control flow patterns or subgraphs are used for optimization, which is difficult to automate and requires developers with strong expertise in optimizing new applications. 

Further, due to the dynamic and runtime nature of past techniques, past work is unable to optimize inter-tensor memory layouts during compilation. Past solutions, thus, either emit gather/scatter operations before and after each batch~\cite{cavs,dynet} or rely on specially designed and/or hand-optimized kernels to operate on scattered data in-place~\cite{cortex,acrobat}, thus precluding the use of highly-optimized vendor libraries on common hardware.

To address these problems, we propose \name~(\underline{E}fficient \underline{D}ynamic \underline{Batch}ing), an efficient automatic batching framework for dynamic neural networks via learned finite state machines (FSM) and batching-aware memory planning. 

For dynamic batching, we exploit the insight that the optimal batching policy for a wide variety of dynamic DNNs can be represented by an FSM, where each state represents a set of possible operator types on the frontier of the dataflow graph.
%at the time of a schedulin through the batching process and makes an independent batching choice. 
Unlike the previous algorithms that depend heavily on aggregated graph statistics to guide batching and result in highly suboptimal decisions, our FSM approach learns which decisions are better by examining the entire graph.
We find that FSMs represents a sweet-spot between expressiveness of batching choices (the same choice for the same state, leveraging the regularity in network topology for a given input) and efficiency. 
%the FSM lever the regularity in network topology better because of its rich possibilities for customizing batch policies for each network structure. 
Further, we adopt a reinforcement-learning (RL) algorithm to learn the FSM from scratch. To guide the training of RL, we design a reward function inspired by a sufficient condition for the optimal batching policy. 

For the static subgraphs of the dynamic DNN, we take a general approach to optimize it by memory-efficient batching. Our key insight is that the memory operations can be significantly minimized by better planning the inter-tensor memory layouts after batching, which we perform using a novel PQ tree-based~\cite{PQTree} algorithm that we have designed.

In summary, this paper makes the following contributions:

\begin{itemize}[topsep=0pt, leftmargin=1.2em, itemsep=-0.5ex]
\item We propose an FSM-based batching algorithm to batch dynamic DNNs that finds near-optimal batching policy.
\item We design a PQ tree-based algorithm with almost linear complexity to reduce memory copies introduced by dynamic batching.
\item We compare the performance of \name with state-of-the-art dynamic DNNs frameworks on eight workloads and achieved on average 1.15x, 1.39x, and 2.45x speedup for chain-based, tree-based, and lattice-based networks across CPU and GPU. We will make the source code for \name publicly available.
\end{itemize}
% \section{Background and Motivation}

% Batching for dynamic neural networks poses the challenge to promote 
% current batching algorithm 
% batching 
% \hl{Todo: describe in detail how batched execution influence the performance.}
% \subsection{Runtime for Dynamic DNN}

% A typical runtime implementation ~\cite{dynet} for dynamic DNN consists of a graph builder, a scheduler, and an executor. First, the user defines the \textit{dataflow graph} for each input instance in a functional language. After that, the scheduler schedules the execution and allocates memory for each operation, which is then executed by the executor. In \name, we focus on the case that the graph builder and scheduler run on the CPU and the executor runs on the GPU. In this way, the scheduler is executed in parallel with the executor. As a result, the overall execution time is estimated by 

% \[T_{build} + \max\{nT_{launch}, T_{execute} + nT_{o}\}\]

% ,where $T_{build}$ is the graph building time, $T_{execute}$ is the execution time for all operations, $T_{launch}$ is the kernel launch overhead for the scheduler on the CPU, $n$ is the number of kernel launches, $T_{o}$ is the constant execution overhead for the executor. In \name, we focus on minimizing $n$ by efficient batching.

\section {FSM-based Algorithm for Dynamic Batching} \label{dynamic-batching}
In this section, we identify the shortcomings of current batching techniques, propose a new FSM-based dynamic batching algorithm, and the mechanism to learn it by RL. 

\subsection {Problem Characterization}

\textit{Dynamic batching} was initially proposed in TensorFlow Fold~\cite{tffold} and DyNet~\cite{dynet-batching} to enable batched execution of operations for dynamic DNNs. Specifically, given a mini-batch of input instances, dataflow graphs are generated for each of the input instances in the mini-batch and each operation is given a type (indicating operation class, tensor shape, etc.). Upon execution, the runtime identifies batching opportunities within the dataflow graphs by executing operations of the same type together. Therefore, the batching algorithm cannot have high complexity to avoid severe runtime overhead. However, the problem of minimizing the number of launched (batched) kernels is an NP-hard problem with no constant approximation algorithm\footnote{Proved by reducing from the \textit{shortest common supersequence} problem ~\cite{SCS} in Appendix.\ref{the: np-hardness for batching}.}, making the batching problem extremely challenging. As a result, the heuristics used for dynamic batching in the current systems often find a suboptimal policy. As shown later (Fig.\ref{fig: kernel launch}), the number of batches executed by current frameworks can be cut by up to 3.27 times.

\begin{algorithm}[tb]
   \caption{FSM-based Dynamic Batching}
   \label{alg:Dynamic Batching}
\begin{algorithmic}[1]
   \State {\bfseries Input:} Dataflow Graph $G$, State Encoding Function $E$, Policy $\pi$;
   \While{G.notEmpty()}
   \State nextType = $\pi(E(G))$
   \State batch = [v for v in Frontier(G) if v.type is nextType]
   \State Execute batch.
   \State Update the Frontier.
   \EndWhile
\end{algorithmic}
\end{algorithm}

\begin{figure}[ht]
\vspace{-1ex}
\centering
\centerline{\includegraphics[width=\columnwidth]{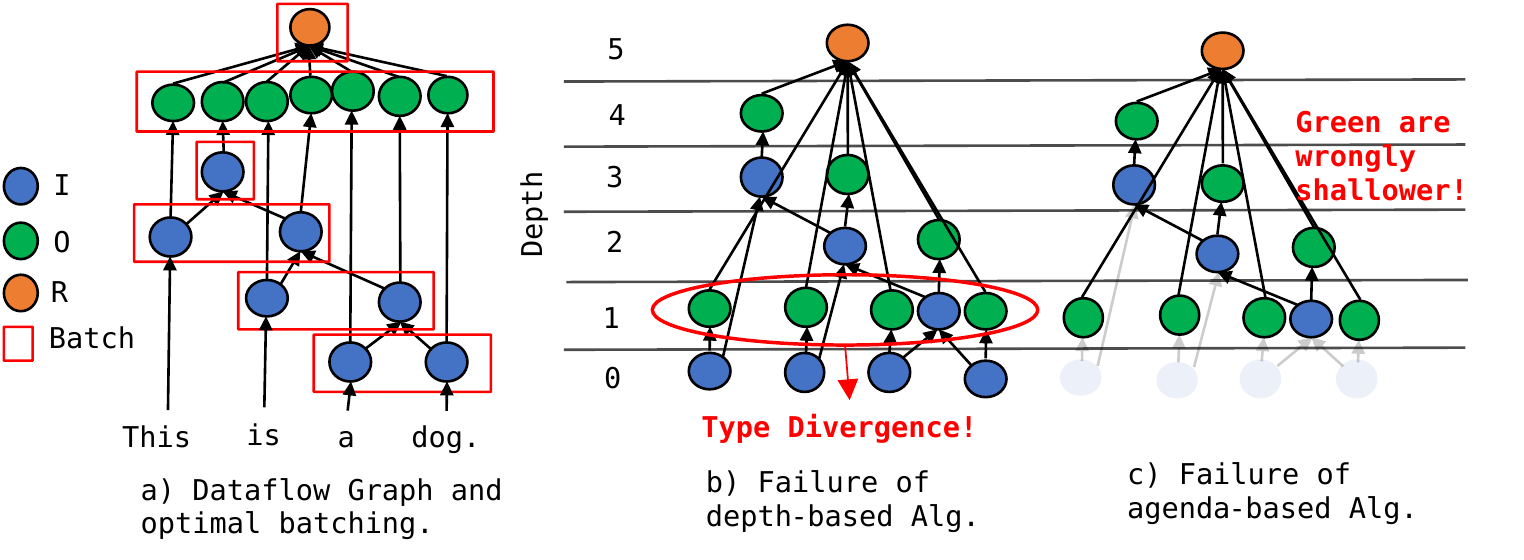}}
\vspace{-0.15in}
\caption{Example on current dynamic batching algorithms.}
\label{fig1:example}
\end{figure}

\begin{figure*}[ht]
\centering
\includegraphics[width=\textwidth]{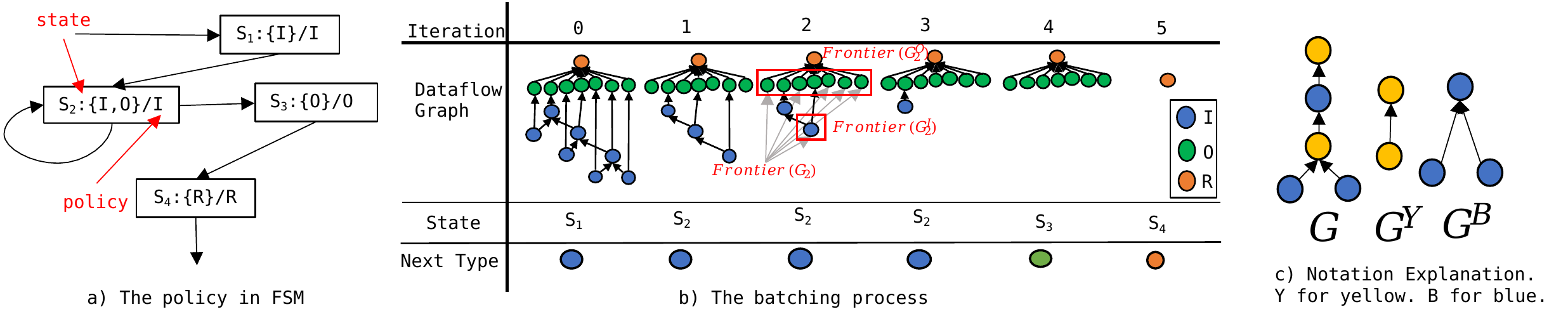}
\vspace{-0.35in}
\caption{Dynamic Batching Policy by FSM}
\vspace{-0.15in}
\label{fig:state machine example}
\end{figure*}

Specifically, previous state-of-the-art algorithms use heuristics depending on aggregated graph statistics to guide batching. The \textit{depth-based algorithm} in Tensorflow Fold~\cite{tffold} batches operations with the same type at the same topological depth (the input operation to the network has depth 0). And the \textit{agenda-based algorithm} in DyNet~\cite{dynet-batching} executes operations of the type with minimal average topological depth iteratively. However, topological depth cannot always capture the regularity of the dataflow graph and result in sub-optimal batching choices. Fig.~\ref{fig1:example}(a) shows a dataflow graph of the tree-based network, which builds upon the parse tree of a sentence with three types of operations: internal nodes ($I$), output nodes ($O$), and reduction nodes ($R$). The ideal batching policy executes all $O$ nodes in one batch. However, the depth-based algorithm in Fig.~\ref{fig1:example}(b) executes the $O$ nodes in four batches because they have different topological depths. For the agenda-based algorithm, consider the scenario when it is deciding the next batch after batching the $I$ nodes first (Fig.~\ref{fig1:example}(c)). Because the $O$ nodes have a lower average depth ($\overline{Depth}=(1+1+1+1+2+3+4)/7=1.85$) than the $I$ nodes ($\overline{Depth}=(1+2+3)/3=2$), the algorithm will pick the $O$ nodes for the next batch, resulting in an extra batch.

\subsection{FSM-based Dynamic Batching}
To fully overcome the limitation of specific graph statistics, we found that an FSM-based approach (1) offers the opportunity to specialize for network structure under a rich design space of potential batching policies and (2) can generalize to any number of input instances, as long as they share the same regularity in topology. 

Shown in Alg.\ref{alg:Dynamic Batching}, the FSM-based dynamic batching approach is an iterative process of choosing the operation type for the next batch. The process differs from the agenda-based algorithm only in how it computes the next type for batching (line 3). During each iteration, the next type is decided by first encoding the current dataflow graph $G$ into a state $S=E(G)$, and then using a policy $\pi$ to map the state into an operation type $t = \pi(S)$. Then, the operations of type $t$ on the frontier of the $G$ form the next batch. After they are executed, these operations are removed from $G$ and the next iteration begins.

For the model in Fig.~\ref{fig1:example}(a), an optimal FSM-based batching policy is shown in Fig.~\ref{fig:state machine example}(a), where we encode the dataflow graph by the set of types on the frontier. Fig.~\ref{fig:state machine example}(b) shows the batching process. From iterations 1 to 3, the dataflow graph is encoded into $S_2=\{I,O\}$, thus the policy continues to batch nodes of type $I=\pi(S_2)$, avoiding batching the $O$ nodes as past heuristics would do. At the same time, it is not hard to see that this FSM-based batching policy can be applied to batch multiple input instances of different parse trees.

\subsection{Using RL to Learn the FSM}
As the FSM provides us with the design space for potential batching policies, we need an algorithm to come up with the best batching policy specialized for a given network structure. In \name, we adopt an RL-based approach for the design-space-exploration and learn the best FSM by a reward function inspired by a sufficient condition for optimal batching. 

 In RL, an agent learns to maximize the accumulated reward by exploring the environment. At time $t$, the environment is encoded into a state $S_t$, and the agent takes action $a_t = \pi(S_t)$ following the policy $\pi$ and receives a reward $r_t = R(S_t, a)$. After this, the environment transforms to the next state $S_{t+1}$. This results in a sequence of states, actions, and rewards: $(S_0, a_0, r_0, S_1, a_1, r_1, ..., S_{N-1}, a_{N-1},r_{N-1}, S_N)$, where $N$ is the number of time steps and $S_N$ is an end state. The agent aims to maximize the accumulated reward $\Sigma_t r_t$ by updating the policy $\pi$. 
 % For Q-learning ~\cite{q-learning}, each state, action pair is given a value $Q(S,a) \in \mathcal{R}$, and the agent chooses the action with the highest value at each state.
For FSM-based dynamic batching, the environment is the dataflow graph, which is encoded into states by the encoding function $E$. For every iteration, the agent decides on the type for the next batch, receives a reward on that decision, and the environment gets updated according to Alg.~\ref{alg:Dynamic Batching}. Now we elaborate on the state encoding, reward design, and training respectively.

Before that, we explain some notations. For a dataflow graph $G$, $G_t$ refers to its status at step $t$, $G^a$ refers to the extracted subgraph of $G$ composed solely of type $a$ operations (See example in Fig.\ref{fig:state machine example}(c)), $Frontier(G)$ refers to the set of ready-to-execute operations, $Frontier_a(G)$ refers to the subset of $Frontier(G)$ with type $a$.

\textbf{State Encoding:} The design of state encoding should be simple enough to avoid runtime overhead. In practice, we experimented with three ways of encoding: (1) $E_{base}(G) = \{v.type | v \in Frontier(G)\}$ is the set of operation types on the frontier, (2) $E_{max}(G) = (E_{base}(G), argmax_{t\in T}|Frontier_{t}(G)|)$ is $E_{base}(G)$ plus the most common type on the frontier and (3) $E_{sort}(G) = sort(\{v.type \in T| v \in Frontier(G)\}, t: |Frontier_{t}(G)|)$ is $E_1(G)$ sorted by the number of occurrences on the frontier. Empirically, we found that $E_{sort}$ was the best among the three (\S\ref{alg evaluation}). 

\textbf{Reward:} We design the reward to minimize the number of batches, thus increasing the parallelism exploited. The reward function is defined as 
\begin{equation}\label{reward function}
r(S_t, a_t) = -1 + \alpha * \frac{|Frontier(G^{a_t}_t)|}{|Frontier_{a_t}(G_t)|}
\end{equation} 

where $\alpha$ is a positive hyper parameter and $S_t = E(G_t)$. The constant $-1$ in the reward penalizes every additional batch, thereby helping us minimize the number of batches.

\begin{lemma}[Sufficient Condition for Optimal Batching]\footnote{Proof in the Appendix~\ref{reward proof}}\label{lemma: sufficient condition}
If $\frac{|Frontier(G^{a_t}_t)|}{|Frontier_{a_t}(G_t)|} = 1$, then there exists a shortest batching sequence starting with $a_t$. 
\end{lemma}

The second term is inspired by a sufficient condition for optimal batching (Lemma \ref{lemma: sufficient condition}) to prioritize the type that all operations in the frontier of the subgraph of this type are ready to execute. For the tree-based network, this term prioritizes the batching choice made by the optimal batching policy in Fig.~\ref{fig:state machine example}(a). For example, at iteration 2, this term is $\frac{5}{7}$ and $\frac{1}{1}$ for the $O$ and $I$ node respectively and the $I$ node is given higher priority for batching. For other networks, like the chained-based networks (Fig.~\ref{fig: kernel launch}), this sufficient condition continues to hold.

\textbf{Training:} We adopt the tabular-based Q-learning~\cite{q-learning} algorithm to learn the policy. And an N-step bootstrapping mechanism is used to increase the current batching choice's influence on longer previous steps. During the training phase, the algorithm learns a $Q$ function, which maps each state and action pair to a real number indicating its score, for every new topology. We found that this simple algorithm learns the policy $\pi$ within hundreds of trials (Table \ref{tab: RL Compile Time}). During inference, at each state $S$, we select the operation type with the highest $Q$ value for the next batch, i.e. $\pi(S) = argmax_a{Q(S, a)}$. This step is done by a lookup into stored $Q$ functions in constant time.

\section{Memory-efficient Batching for Static Subgraphs}\label{BasicBlockOpt}

\begin{figure}
\vskip 0.2in
\centering
\centerline{\includegraphics[width=\columnwidth]{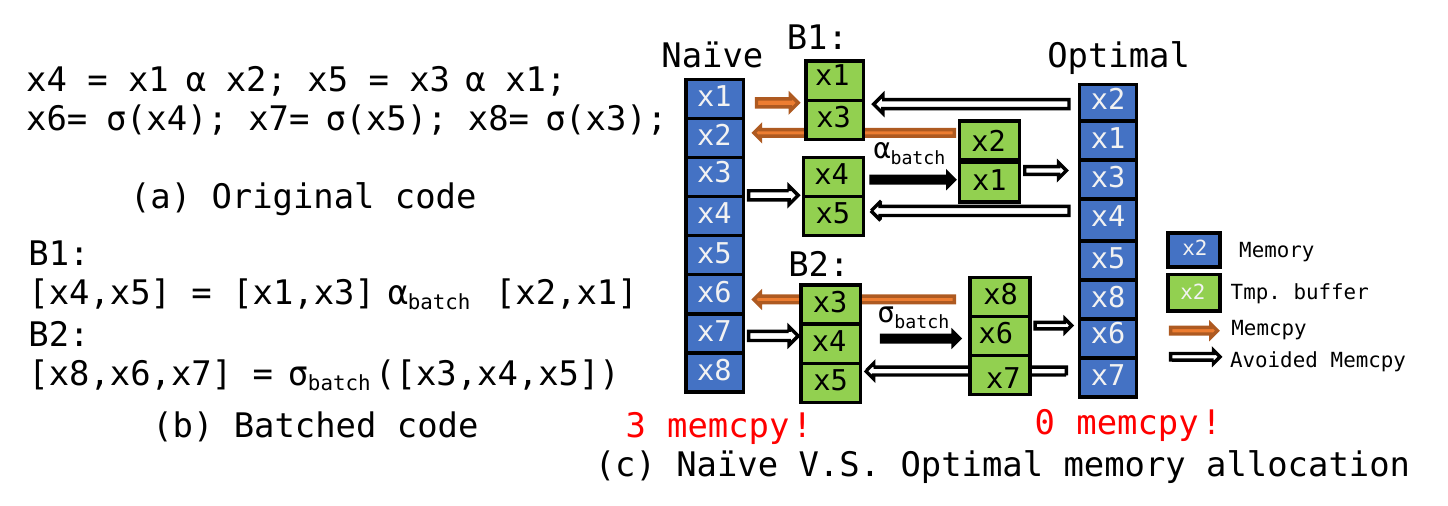}}
\vspace{-0.15in}
\caption{Example on memory allocation. $\alpha$, $\sigma$ represent operators.}
\label{fig: bb-opt example}
\end{figure}

\subsection{Background and Motivation}
In order to invoke a batched tensor operator in a vendor library, the source and result operands in each batch are usually required to be contiguous in memory (as per the vendor library specifications). Current batching frameworks such as Cavs and DyNet ensure this by performing explicit memory gather and/or scatter operations, leading to high data movement. On the other hand, as mentioned above, Cortex relies on specialized, hand-optimized batched kernels instead of relying on vendor libraries. This approach however, is unable to reuse the highly performant optimizations available as part of vendor libraries. 

In \name, we take a different approach to fit the memory layout into the batching policy, where operations in the source and result operands for batched execution are already contiguous in memory. 

We illustrate the approach by the example. Fig.~\ref{fig: bb-opt example}(a) shows an example code for a static subgraph and Fig.~\ref{fig: bb-opt example}(b) shows its batched version. In Fig.~\ref{fig: bb-opt example}(c), we compare two memory layouts. On the left, we directly allocate memory according to the variable's label, then two memory gather for $[x_1,x_3], [x_2, x_1]$ and one scatter for $[x_8, x_6, x_7]$ is performed because they are either not contiguous or aligned in memory. We say an operand of a batch is aligned in memory if the order of its operations matches with the one in memory. Now, consider the memory allocation on the right, which allocates memory following $(x_2,x_1,x_3,x_4,x_5,x_8,x_6,x_7)$. Then, every source and result operand of the batched execution is already contiguous and aligned in memory, saving us from extra memory copies. 

\subsection{PQ tree-based memory allocation}
To find the ideal layout, we designed an almost linear complexity memory allocation algorithm based on PQ tree, which is a tree-based structure used to solve the consecutive one property~\cite{c1p} and is previously applied to Biology for DNA-related analysis~\cite{pqtree-bio-1}. 

We define the \textit{ideal memory layout} as a sequence of variables satisfying two constraints:
\begin {itemize}[noitemsep,topsep=0pt,parsep=0pt,partopsep=0pt]
\item \textbf{Adjacency Constraint:} Result and source operands in every batch should be adjacent in the sequence. E.g. $\{x_4,x_5\}, \{x_1,x_3\}, \{x_2, x_1\}$ for $B1$, $\{x_6,x_7,x_8\}, \{x_4,x_3,x_5\}$ for $B2$ are adjacent in the sequence. 
\item \textbf{Alignment Constraint:} The order of the result and source operands should be aligned in a batch. E.g. for $B1$,  $x_4\prec x_5 \iff x_1 \prec x_3 \iff x_2 \prec x_1$ in the sequence.
\end {itemize}

\begin{figure*}[ht]
\centering
\vspace{-1ex}
\centerline{\includegraphics[width=\textwidth]{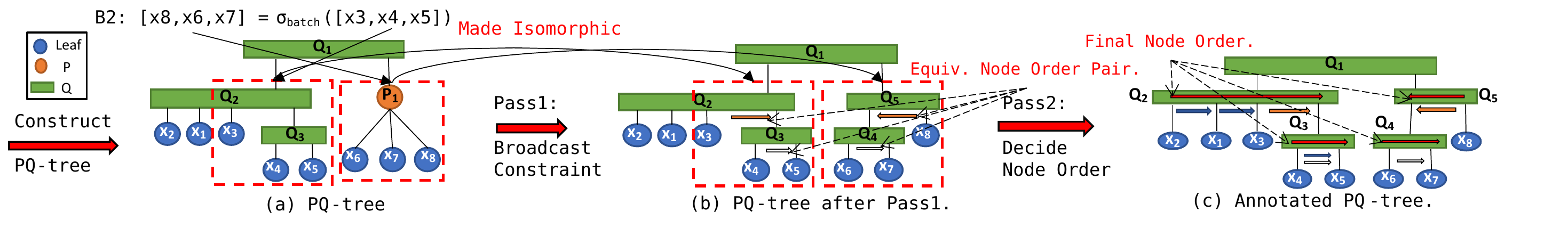}}
\vspace{-2mm}
\caption{Example for PQ tree-based algorithm. $x_1$-$x_8$ are variables. }
\label{fig: pqtree-example}
% \vspace{-1ex}
\end{figure*}

The adjacency constraint is satisfied by PQ tree~\cite{PQTree} algorithm. 
Given several subsets of a set $S$, the PQ tree algorithm returns a data structure in linear time called PQ tree, representing potential permutations of $S$ that elements in each subset are consecutive. Fig.\ref{fig: pqtree-example}(a) shows the PQ tree for the example code. The tree has three kinds of nodes: P-node, Q-node, and leaf node. Leaf nodes represent the variables; P-nodes have more than two children, whose appearance in the sequence is contiguous but could be permuted; Q-nodes have more than one child, whose appearance in the sequence follows an order but could be reversed. A depth-first traversal of the leaf nodes gives the sequence. For example, there is one P-nodes and three Q-nodes in Fig.\ref{fig: pqtree-example}(a). $Q_2$ indicates the order should only be $(x_2,x_1,x_3,Q_3)$ or $(Q_3,x_3,x_1,x_2)$, while $P_2$ indicates that one permutation of $\{x_6, x_7, x_8\}$ appears in the sequence. The adjacency of $\{x_4,x_5\}$ is embedded in $Q_3$, $\{x_1,x_3\}, \{x_2,x_1\}, \{x_4,x_3,x_5\}$ in $Q_2$, and $\{x_6,x_7,x_8\}$ in $P_1$. A possible sequence is $(x_2,x_1,x_3,x_4,x_5,x_6,x_7,x_8)$.

To satisfy the alignment constraint, we annotate each node on the PQ tree with an order. An annotated PQ tree is shown in Fig.\ref{fig: pqtree-example}(c), where a direction mark is attached to every Q-node, indicating its traversal order. As a result, any leaf node sequence of legal traversal on this annotated PQ tree indicates a memory allocation order satisfying the alignment constraint. 

\begin{algorithm} [tb]
\scriptsize
\caption{PQ tree Memory Allocation}
\label {alg: PQ tree Memory Arrangement}
\begin{algorithmic}[1]
    \Function{BroadcastConstraint}{$tree$, $\mathcal{B}$}\label{alg:pqtree:pass1}
        \State visited = getSet()
        \For {$batch$ in $\mathcal{B}$}
            \If {batch in visited}
                \State continue
            \EndIf 
            \State Q = Queue()
            \State Q.push($batch$)
            \While {not Q.isEmpty()} \label{alg:pqtree:iteration}
                \State $b$ = Q.pop();
                \State visited.insert($b$)
                \State $cons = $ \Call{ParseConstraints}{$b$}\label{alg:pqtree:parseconstraint}
                \State $suc,updatedBatches = $\Call{ApplyConstraints}{$cons$,$tree$}\label{alg:pqtree:applyconstraints}
                \If {$suc$ is False}    
                    \State $\mathcal{B}$.erase(b)
                \Else
                    \For {$b$ in $updatedBatches$}
                        \State Q.push($b$)
                    \EndFor 
                \EndIf
            \EndWhile
        \EndFor
    \EndFunction

    \Function{DecideNodesOrder}{$tree$, $\mathcal{B}$}\label{alg:pqtree:pass2}
        \State $POrder$ = getUnionFindSet(tree.PNodes) \Comment{A union-find set to decide QNode's direction.}
        \State $QOrder$ = getUnionFindSet(tree.QNodes) \Comment{A union-find set to decide PNode's permutation.}
        \For {$batch$ in $\mathcal{B}$}
            \State $EquivPairs$ = \Call{ParseEquivNodeOrderPair}{$tree$, $batch$}\label{alg:pqtree:parse equi-rel}
            \For {$EquivPair$ in $EquivPairs$}\label{alg:pqtree:spread equi-rel}
                \If {$EquivPair$ is a P-node pair}
                    \State POrder.Union($EquivPair$)
                \ElsIf {$EquivPair$ is a Q-node pair}
                    \State QOrder.Union($EquivPair$)
                \EndIf 
            \EndFor
        \EndFor
        \State \textbf{return} $QOrder$, $POrder$
    \EndFunction 
    
    \Function{Main}{$X$, $\mathcal{B} = [batch_1, ..., batch_n]$} 
        \Comment{$X$ the variable set, $\mathcal{B}$ the batches }
        \State $tree =$ \Call{ConstructPQTree}{$X$, $\mathcal{B}$}
        \State \Call{BroadcastConstraint}{$tree$, $\mathcal{B}$}
        \State $QOrder$, $POrder$ = \Call{DecideNodesOrder}{$tree$, $\mathcal{B}$}
        \State \textbf{return} \Call{GetLeafOrder}{$tree$, $QOrder$, $POrder$}
    \EndFunction
\end{algorithmic}
\end {algorithm}

Two passes obtain the order annotation to the PQ tree (Fig.\ref{fig: pqtree-example}, Alg.\ref{alg: PQ tree Memory Arrangement}). The first pass, \textproc{BroadcastConstriant}, makes the tree structure of each batch's operands isomorphic. For $B2$'s operands, $\{x_3,x_4,x_5\}$'s tree structure is $Q_2=(...,x_3, Q_3=(x_4,x_5))$, and $\{x_6,x_7,x_8\}$'s tree structure is $P_1=(x_6,x_7,x_8)$. They are made isomorphic in this pass. The second pass, \textproc{DecideNodeOrder}, derives the equivalent class of node and order pairs among P and Q-nodes and their orders and searches for a compatible solution for the direction of the Q-node and a permutation choice of the P-node complied with the equivalence relationship. 

We walk through the algorithm in the example. At first, the PQ tree is constructed by the standard algorithm to satisfy the adjacency constraint (Fig.\ref{fig: pqtree-example}(a)). After that, in \textproc{BroadcastConstraint}, for $B2$, we parse the adjacency constraint of it (line \ref{alg:pqtree:parseconstraint}) by 1) parsing the adjacency constraint for each operand, e.g. $\{x_3, x_4,x_5\}$ and $\{x_4,x_5\}$ for the source operand, $\{x_6,x_7,x_8\}$ for the result operand, and 2) transforming them across operands by alignment information, e.g. $\{x_4,x_5\}$ is transformed into $\{x_6,x_7\}$. After that, $\{x_6,x_7\}$ is applied to the PQ tree\footnote{Perform by standard \textproc{Reduce} step in the Vanilla PQ tree algorithm to satisfy adjacency constraint by restructuring the tree. } and we keep records on the batch whose tree structure has changed (line \ref{alg:pqtree:applyconstraints}). Now tree structures for $B2$'s operands are isomorphic, and we apply this process to other batches in a breadth-first search until no update on the tree structure happens. 

\begin{figure}
\vspace{-0.2in}
\centering
\centerline{\includegraphics[width=\columnwidth]{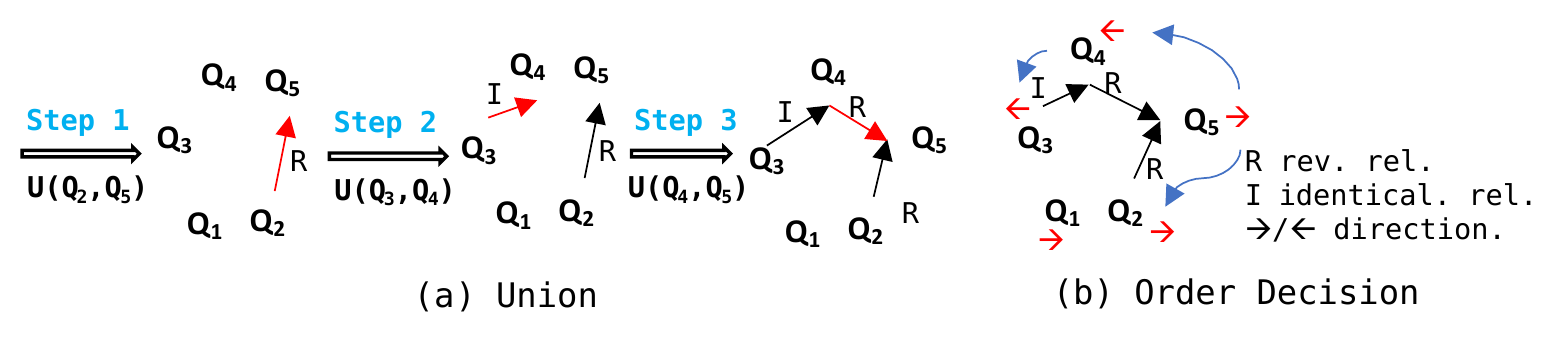}}
\vspace{-0.2in}
\caption{Illustration for order decision in Pass2.}
\label{fig: pqtree pass2}
\end{figure}

In \textproc{DecideNodeOrder} (line \ref{alg:pqtree:pass2}), we assign directions for the Q-node and the permutation for the P-node. We start by parsing the equivalence relationship (line \ref{alg:pqtree:parse equi-rel}) among Q-node and direction pairs or the P-node and permutation pairs from the isomorphic tree structures after the first pass, e.g. $<Q_2,\leftarrow>\iff<Q_3,\leftarrow>$ for $B1$, $<Q_3,\leftarrow>\iff<Q_4,\leftarrow>$ and $<Q_2,\leftarrow>\iff <Q_5,\rightarrow>$ for $B2$. After that,   
we spread the equivalence relationship across batches with the support of \textit{union–find data structure}. Shown in Fig.\ref{fig: pqtree pass2}, a graph carrying the equivalence relationship is constructed by iteratively \textproc{Union} equivalent relationship (line \ref{alg:pqtree:spread equi-rel}), i.e. the node and order pairs. When processing $<Q_2,\leftarrow>\iff <Q_5,\rightarrow>$, a $R$ edge $<Q_2,Q_5>$ is added to the graph, indicating $Q_2$'s direction is determined by the reverse of $Q_5$'s direction. 
When processing the $<Q_2,\leftarrow>\iff<Q_3,\leftarrow>$, we first find the decider of their order, i.e. $Q_5$ for $Q_2$ and $Q_4$ for $Q_3$, and add an $R$ edge between them. In this way, $Q_2$ and $Q_3$ always have the same order. Finally, the deciders in the graph are assigned arbitrary directions, which spread across the graph following the relationship on the edge (Fig.\ref{fig: pqtree pass2}(b)).

% TODO: make it more clear in the future. 
The PQ tree memory allocation algorithm's time complexity is given in lemma \ref{lemma:pqtree:complexity}, showing the PQ tree algorithm is linear to the problem size if the operations on a single batch are limited to a constant. 
\begin{lemma}\label{lemma:pqtree:complexity}
PQ tree memory allocation algorithm's time complexity is $O(\Sigma_{b\in batches}|b|\max^2_{b\in batches}|b|)$, where $|.|$ counts the operation in a batch.
\end{lemma}

Right now, PQ tree memory optimization is applied to the static subgraph because its execution time still does not fit into the high runtime constraint for dynamic DNNs. But the algorithm itself and the idea of better memory planning for batching is applicable to any batching problem. A detailed explanation of the algorithm is in Appendix\ref{ref:pqtree}.

\section{Implementation}
% Currently, the user actually needs to explicitly define the block using dynet's functional language
The optimizations in \name~are fully automated and implemented as a runtime extension to DyNet in 5k lines of C++ code. The user can optionally enable the batching optimization by passing a flag when launching the application, and enable the static subgraph optimization by defining subgraphs in DyNet's language with a few annotations. Before execution, the RL algorithm learns the batching policy and \name~optimizes the static subgraph by the approach in \S\ref{BasicBlockOpt}. Upon execution, \name~calls DyNet's executor for batched execution, which is supported by vendor libraries.
\section {Evaluation}

\subsection{Experiment Setup}
We evaluate our framework against DyNet and Cavs, two state-of-the-art runtimes for dynamic DNNs, which are shown to be faster than traditional frameworks like Pytorch and Tensorflow ~\cite{cavs, dynet}. Cavs' open-sourced version has worse performance than DyNet, as stated in ~\citet*{cortex} because certain optimizations are not included. To make a fair comparison with Cavs, we use an extended version of DyNet with the main optimizations in Cavs enabled as a reference for Cavs' performance (referred to as Cavs DyNet). Namely, the static subgraphs in the network are pre-defined and batching optimization is applied to them. \name~is implemented on top of this extended version, with the RL-based dynamic batching algorithm ($E_{sort}$ for state encoding) and memory optimization on the static subgraph by PQ Tree. On the other side, the agenda-based algorithm and the depth-based algorithm are used for dynamic batching on Vanilla/Cavs DyNet. Depending on the workload and configuration, a better-performing algorithm is chosen for Vanilla/Cavs DyNet in the evaluation. 

We test the framework on 8 workloads, shown in Table~\ref{tab:workloads}. They follow an increase in dynamism, from chains to trees and graphs. Except for lattice-based networks, all workloads appeared as the benchmark for past works. 

We run our experiments on a Linux server with an Intel Xeon E2590 CPU (28-physical cores) and an Nvidia V100 GPU. The machine runs CentOS 7.7, CUDA 11.1, and cuDNN 8.0. We use DyNet's latest version (Aug 2022, commit c418b09) for evaluation.

\begin{table}[t]
\centering
\caption{Models and datasets used in evaluation}
\scriptsize
\begin{tabular}{ p{14em} | p{5em} | p{7em} }
\toprule[2pt]
\textbf{Model} & \textbf{Short name} & \textbf{Dataset} \\
\hline 
A bi-directional LSTM Named Entity Tagger \cite{bilstm-tagger} & BiLSTM-Tagger & WikiNER English Corpus\cite{bilstm-dataset} \\\hline
An LSTM-based encoder-decoder model for neural machine translation. & LSTM-NMT & IWSLT 2015 En-Vi \\\hline
N-ary TreeLSTM~\cite{treelstm} & TreeLSTM & \multirow{3}{5em}{Penn treebank~\cite{penn-treebank}} \\ \cline{1-2} 
N-ary TreeGRU & TreeGRU & \\ \cline{1-2} 
MV-RNN~\cite{MVRNN} & MV-RNN & \\  \cline{1-2}
An extension to TreeLSTM that contains two types of internal nodes, each with 50\% probability & TreeLSTM-2Type & \\
\hline
A lattice-based LSTM network for Chinese NER ~\cite{LatticeLSTM} & LatticeLSTM  & \multirow{2}{5em}{Lattices generated based on Chinese Weibo Dataset\footnote{https://github.com/OYE93/Chinese-NLP-Corpus/tree/master/NER/Weibo}}\\ \cline{1-2}
A lattice-based GRU network for neural machine translation~\cite{latticeGRU} & LatticeGRU  & \\ 
\bottomrule[2pt]
\end{tabular}
\label{tab:workloads}
\end{table}

\begin{figure*}[!t]
\vspace{-1ex}
     \centering
     \includegraphics[width=\textwidth]{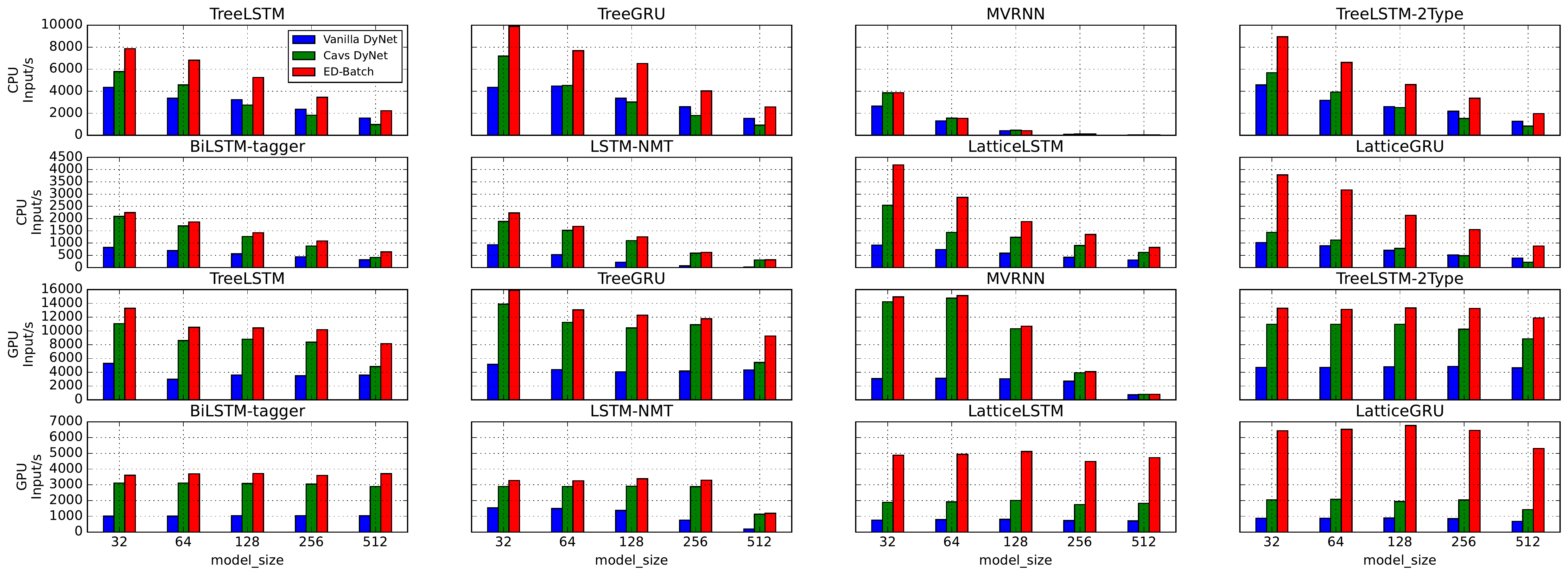}
    \vspace{-8mm}
        \caption{\name~v.s. Vanilla/Cavs DyNet: Inference Throughput}
        \label{fig: overall performance}
    \vspace{-2ex}
\end{figure*}

\subsection{Overall Performance}
We compare \name's end-to-end inference throughput against Vanilla/Cavs DyNet. We follow past work to evaluate different batch sizes (1, 8, 32, 64, 128, 256) and model sizes (32, 64, 128, 256, 512), which is the size for the hidden vector length and the embedding size. The throughput is calculated as the maximum throughput among all bath size choices. For all cases, \name~outperforms Vanilla DyNet significantly due to the reduction in graph construction and runtime overhead by pre-definition of the static subgraph. 

We now discuss the comparison with Cavs DyNet. For chain-based models, the BiLSTM-tagger and LSTM-NMT, \name~achieved on average 1.20x, 1.11x speedup on CPU and 1.20x, 1.12x on GPU. Because the network structure is basically made up of chains, shown in Fig.\ref{fig: kernel launch},  both the agenda-based algorithm and the FSM-based batching algorithm find the optimal batching policy. On the other hand, the LSTMCell is 1.54x faster with the PQ-tree optimization compared to the one with the DyNet's memory allocation, which explains the speedup.

For the tree-based model, compared to agenda/depth-based batching heuristic \name~reduces the number of batches by 37\%. This is because the FSM-based algorithm executes the output nodes in one batch (Fig.\ref{fig1:example}). For TreeLSTM and TreeGRU, \name~achieved on average 1.63x, 1.46x speedup on CPU and 1.23x, 1.29x speedup on GPU. \name's performance is close to Cavs DyNet on MVRNN because the execution is bounded by matrix-matrix multiplications which can hardly benefit from extra batch parallelism and the reduction in runtime overhead. 

For the lattice-based models, the LatticeLSTM and LatticeGRU, \name~increases DyNet Cavs's throughput significantly by 1.32-2.97x on CPU and 2.54-3.71x on GPU, which is attributed to both the better dynamic batching and static subgraph optimization. For the lattice-based models' network structure in Fig.\ref{fig: Lattice}, the FSM-based algorithm prioritizes the execution of the character cell and delays the execution of the word cell, whereas the depth/agenda-based algorithms batch the character cell and word cell more arbitrarily. As a result, the number of batches is reduced by up to 3.27 times (Fig.\ref{fig: kernel launch}). For the static subgraph, the used LSTMCell and GRUCell's latency is cut by 34\% and 35\%, which adds to the speedup.

\subsection{Analysis}

\textbf{Where does \name's speedup come from?} Shown in Fig.\ref{fig: time decomposition}, we decompose the inference pass into the construction time, scheduling time, and execution time. Construction time is the time to define the dataflow graph. Scheduling time is the time for dynamic batching analysis. Execution time is the left time of the forward pass, mainly composed of the execution of operations. While having similar construction/scheduling time, \name~speeds up Cavs DyNet in the great cutdown in execution time benefited from better batching and fewer kernels for data movement. 

\begin{figure}[ht]
\centerline{\includegraphics[width=\columnwidth]{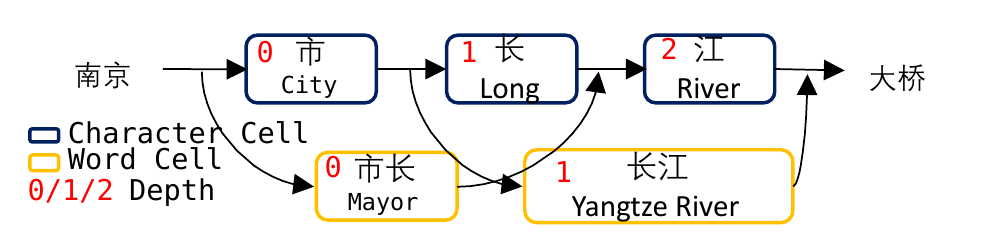}}
\caption{Lattice Network for Chinese NER. The topology for one input sentence is a chain of character cells with jump links of word cells. Agenda/Depth-based batching algorithm fails to batch the word cells together. }
\label{fig: Lattice}
\end{figure}

\begin{figure}[ht]
\vspace{-2ex}
\centering
\centerline{\includegraphics[width=0.95\columnwidth]{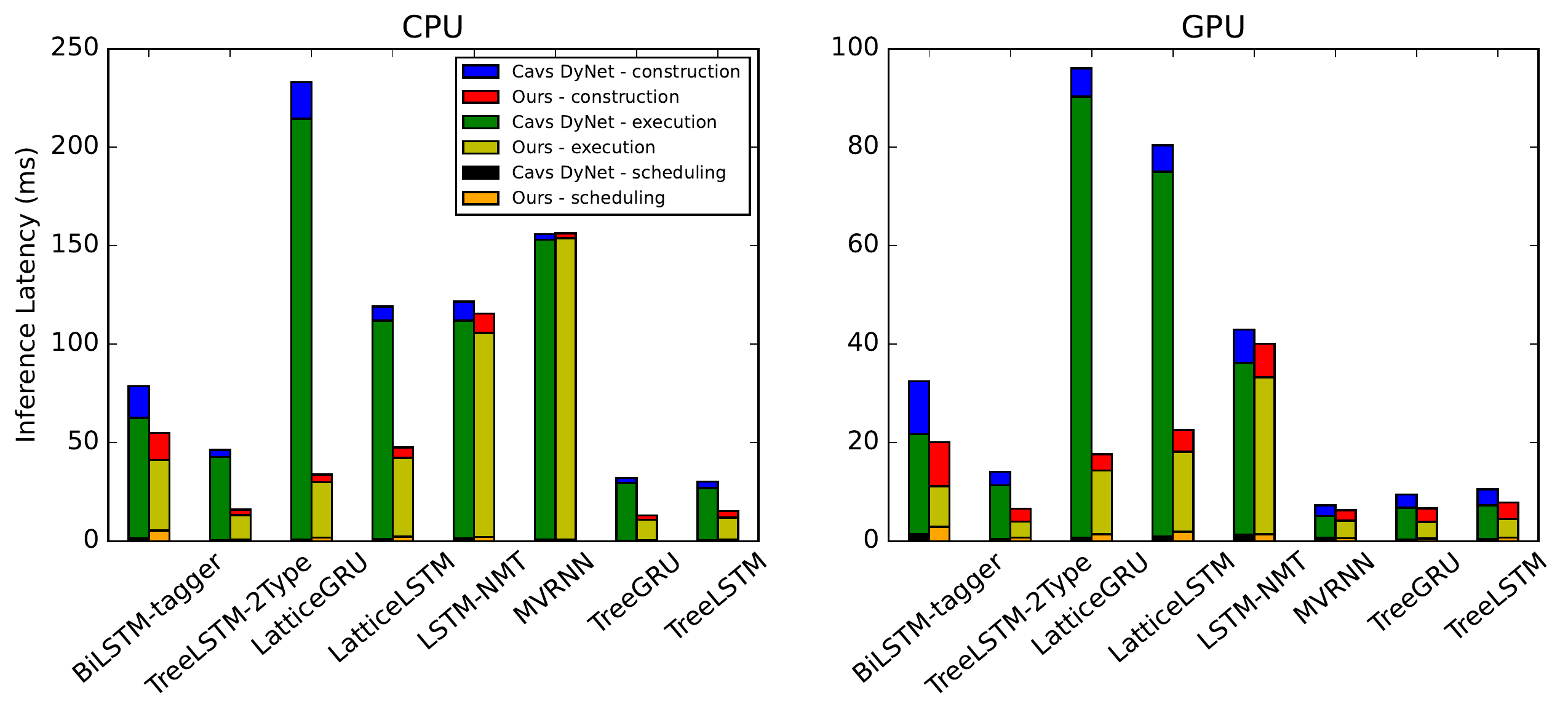}}
\vspace{-2ex}
\caption{Cavs DyNet v.s. \name: Time Decomposition when model size is 128 and batch size is 64.}
\label{fig: time decomposition}
\vspace{-1ex}
\end{figure}

\begin{figure}[ht]
\vspace{-2ex}
\centering
\centerline{\includegraphics[width=0.8\columnwidth]{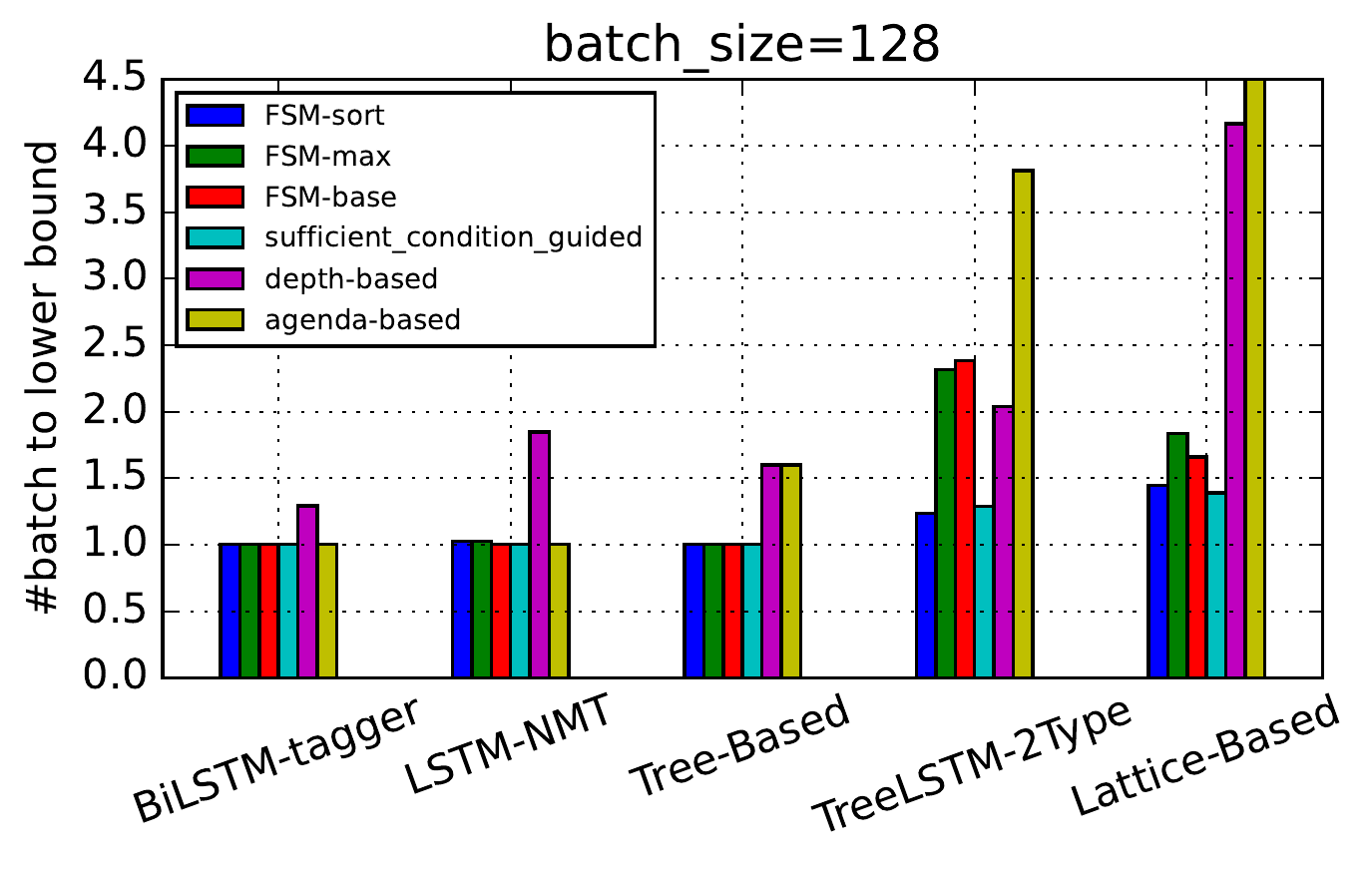}}
\vspace{-2ex}
\caption{The number of batches for different batching algorithms. FSM-base/sort/max refers to the FSM based algorithm with different state encodings.}
% \vskip -0.3in
\label{fig: kernel launch}
% \vspace{ex}
\end{figure}

\begin{table}[t]
\caption{Batching with DyNet's memory allocation (left) v.s. Batching with PQ tree-based memory allocation on static subgraphs. batch size = 8, model size = 64.}
\scalebox{0.60}{
\begin{tabular}{lllllll}
\toprule
{} & \multicolumn{2}{l}{Latency (ms)} & \multicolumn{2}{l}{Mem Kernels/Subgraph} & \multicolumn{2}{l}{Memcpy Amount (kB)} \\
Subgraph &                 value & ratio &               value & ratio &                   value &  ratio \\
\midrule
GRUCell           &  0.11 / \textbf{0.07} &  1.54 &  6 / \textbf{2} &   3.0 &   666.0 / \textbf{14.0} &  47.57 \\
LSTMCell          &   0.2 / \textbf{0.13} &  1.52 &  4 / \textbf{1} &   4.0 &  1054.0 / \textbf{16.0} &  65.88 \\
MVCell            &  \textbf{0.08} / 0.08 &  0.96 &  2 / \textbf{2} &   1.0 &  260.0 / \textbf{260.0} &    1.0 \\
TreeGRU-Internal  &  0.24 / \textbf{0.15} &   1.6 &  8 / \textbf{2} &   4.0 &   552.0 / \textbf{16.0} &   34.5 \\
TreeGRU-Leaf      &  0.09 / \textbf{0.07} &   1.4 &  4 / \textbf{2} &   2.0 &    268.0 / \textbf{8} &   33.5 \\
TreeLSTM-Internal &  0.19 / \textbf{0.12} &  1.61 &  7 / \textbf{3} &  2.33 &  1064.0 / \textbf{22.0} &  48.36 \\
TreeLSTM-Leaf     &  0.12 / \textbf{0.09} &  1.27 &  3 / \textbf{1} &   3.0 &    396.0 / \textbf{6.0} &   66.0 \\
\bottomrule
\end{tabular}
}

\label{tab: subgraph}
% \vspace{-3ex}
\end{table}

\begin{table}
\scalebox{0.65}{
\parbox{.9\columnwidth}{
\caption{RL Training Time and iterations}
\centering
\begin{tabular}{lrr}
\toprule
{} &    Time (s) &  Train Iter. \\
\midrule
TreeLSTM               &   0.154 &        50 \\
TreeGRU                &   0.141 &        50 \\
MVRNN                  &   0.254 &        50 \\
TreeLSTM-2type         &   2.217 &      1000 \\
BiLSTM-tagger          &   1.629 &        50 \\
BiLSTM-tagger-withchar &   6.268 &        50\\
LatticeLSTM            &  21.733 &      1000 \\
LatticeGRU             &   4.911 &      1000 \\
\bottomrule
\end{tabular}
\label{tab: RL Compile Time}
}
\hspace{1ex}
\parbox{.5\columnwidth}{
\centering
\caption{Static Subgraph Compile Time}
\begin{tabular}{lc}
\toprule
{} &  Time (ms) \\
\midrule
GRUCell           &               4.82 \\
LSTMCell          &              12.95 \\
MVCell            &               1.53 \\
TreeGRU-Internal  &              10.43 \\
TreeGRU-Leaf      &               2.91 \\
TreeLSTM-Internal &              29.89 \\
TreeLSTM-Leaf     &               3.64 \\
\bottomrule
\end{tabular}
\label{tab: static subgraph Compile Time}
}
}
\end{table}

\begin{table}[]
    \centering
    \caption{\name~v.s. Cortex: Inference Latency (ms).}
\scalebox{0.8}{
    \begin{tabular}{llrrrr}
\toprule
   &  & \multicolumn{2}{l}{TreeGRU} & \multicolumn{2}{l}{TreeLSTM} \\
 batch\_size  & model\_size  &  Cortex &  Ours &   Cortex &  Ours \\
\midrule
10 & 256 &    2.30 &  \textbf{2.27} &    \textbf{2.244} &  2.78 \\
   & 512 &    5.60 &  \textbf{3.04} &    8.500 &  \textbf{4.70} \\
20 & 256 &    3.73 &  \textbf{3.03} &    \textbf{3.460} &  3.52 \\
   & 512 &   11.70 &  \textbf{3.70} &   19.210 &  \textbf{4.82} \\
\bottomrule
\end{tabular}
}
    \label{tab:cortex}
\end{table}

\textbf{Does the algorithm find a good enough batching policy?}\label{alg evaluation} Shown in Fig.\ref{fig: kernel launch}, compared to the agenda/depth-based batching algorithm, \name's FSM-based batching algorithm uniformly executes fewer batches. Among three state encoding choices, $E_{sort}$ is slightly better because of the stronger expressiveness, finding the optimal batching policy on BiLSTM-tagger, LSTM-NMT, and Tree-based models and executing 23\% and 44\% more batches on TreeLSTM-2Type, Lattice-Based models.  

To demonstrate the efficiency of the reward function, we measure the number of batches executed by a sufficient-condition-guided heuristic, which selects the type for the next batch that maximizes the second term in Eq.\ref{reward function}. Shown in Fig.~\ref{fig: kernel launch}, this heuristic executes batches paramount to the best FSM-based algorithm. However, this heuristic has higher time complexity and adds to unacceptable runtime overhead. Thus, on the evaluated workloads, the FSM-based algorithm can be treated as a time-efficient distiller of this heuristic. 

\textbf{Ablation Study of the Static Subgraph Optimization.} In Table \ref{tab: subgraph}, we evaluate \name's memory layout optimization on the static subgraphs. For all evaluated cases, the PQ-tree algorithm finds the \textit{ideal memory allocation order} (Remained data transfer is caused by broadcast that cannot be optimized by better memory layout). Compared to the baseline, \name reduces the latency of the static subgraph by up to 1.6x, memory kernels by up to 4x, and memory transfer amount by up to 66x. This significant reduction in memory transfer can be attributed to the better arrangement of the weight parameters. For example, there are four gates in the LSTM cell that perform feed-forward arithmetic $y_i=W_ix_i+b_i$, which are executed in a batch. The memory arrangement in \name~makes sure the inputs, parameters, and intermediate results of batched kernels are contiguous in the memory, which is not considered by DyNet's policy. Since the weight matrix occupies memory relative to the square of the problem size, this leads to a huge reduction in memory transfer.

\textbf{Comparison with a more specialized framework.} Cortex~\cite{cortex} is highly specialized for optimizing a class of recursive neural networks and it requires the user to not only express the tensor computation, but also specify low-level optimizations specific to underlying hardware, both through TVM's domain-specific language~\cite{tvm}. We compare \name with Cortex on TreeLSTM and TreeGRU. To make more of an apples-to-apples comparison in terms of the user's effort in developing the application, we enabled Cortex's automated optimizations like \textit{linearization} and \textit{auto-batching} and used simple policies on optional user-given (manual) optimizations like kernel fusion and loop transformation (details in \S\ref{sec:appendix:cortex}).  As shown in Table \ref{tab:cortex}, \name can speed up Cortex by up to 3.98x. 

\textbf{The compilation overhead.} We trained the RL for up to 1000 trials and stopped early if the number of batches reaches the lower bound (check every 50 iterations). On the static subgraph, batching is performed as a grid search and the PQ tree optimization is applied afterward. As shown in Table~\ref{tab: RL Compile Time} and Table~\ref{tab: static subgraph Compile Time}, it takes tens of milliseconds to optimize the static subgraph and up to 22 seconds to learn the batching policy for tested workloads.
\section{Related Work}
% \textbf{Frameworks for Dynamic Neural Networks.} 
There are a variety of frameworks specialized for dynamic neural network training~\cite{dynet,tffold,cavs}, inference~\cite{cortex,acrobat,jit-batching}, and serving~\cite{batchmaker}. Concerning batching for the dynamic neural networks, DyNet~\cite{dynet} and TFFold~\cite{tffold} laid the system and algorithm foundation to support \textit{dynamic batching}. However, their batching heuristics are often sub-optimal as we saw above. Nevertheless, their algorithms have been used in other frameworks, like Cavs~\cite{cavs} and Acrobat~\cite{acrobat}. Apart from batching, another major direction of optimization is to extract the static information from the dynamic DNN and optimize them during compile time. Cavs~\cite{cavs} proposed the idea of predefining the static subgraphs, which was later extended in \cite{jit-batching} to batch on different granularities. \name~adopts this multi-granularity batching idea to perform batching on both the graph level and the subgraph level. For static subgraphs, traditional techniques are used for optimization, like the kernel fusion in Cavs and Cortex~\cite{cortex}, the AoT compilation~\cite{nimble}, and specialized kernel generation in ACRoBat~\cite{acrobat}. However, though with high efficiency, these optimizations can hardly be automated because either the developer or the user needs to optimize each subgraph manually. In \name, fully automated runtime optimizations are used instead to enable both efficient execution and generalization.

\section {Conclusion}
In \name, we designed an FSM-based algorithm for batched execution for dynamic DNNs. Also, we mitigated the memory transfer introduced by batching through the memory layout optimization based on the PQ-tree algorithm. The experimental results showed that our approach achieved significant speedup compared to current frameworks by the reduction in the number of batches and data movement.
\section*{Acknowledgements}
We thank Peking University's supercomputing team for providing the hardware to run the experiments and thank Kevin Huang of CMU for his help in early exploration of this research space. 

\bibliography{refs}
\bibliographystyle{icml2020}

% %%%%%%%%%%%%%%%%%%%%%%%%%%%%%%%%%%%%%%%%%%%%%%%%%%%%%%%%%%%%%%%%%%%%%%%%%%%%%%%
% %%%%%%%%%%%%%%%%%%%%%%%%%%%%%%%%%%%%%%%%%%%%%%%%%%%%%%%%%%%%%%%%%%%%%%%%%%%%%%%
% % DELETE THIS PART. DO NOT PLACE CONTENT AFTER THE REFERENCES!
% %%%%%%%%%%%%%%%%%%%%%%%%%%%%%%%%%%%%%%%%%%%%%%%%%%%%%%%%%%%%%%%%%%%%%%%%%%%%%%%
% %%%%%%%%%%%%%%%%%%%%%%%%%%%%%%%%%%%%%%%%%%%%%%%%%%%%%%%%%%%%%%%%%%%%%%%%%%%%%%%
% \appendix
% \section{Do \emph{not} have an appendix here}

% \textbf{\emph{Do not put content after the references.}}
% %
% Put anything that you might normally include after the references in a separate
% supplementary file.

% We recommend that you build supplementary material in a separate document.
% If you must create one PDF and cut it up, please be careful to use a tool that
% doesn't alter the margins, and that doesn't aggressively rewrite the PDF file.
% pdftk usually works fine. 

% \textbf{Please do not use Apple's preview to cut off supplementary material.} In
% previous years it has altered margins, and created headaches at the camera-ready
% stage. 
%%%%%%%%%%%%%%%%%%%%%%%%%%%%%%%%%%%%%%%%%%%%%%%%%%%%%%%%%%%%%%%%%%%%%%%%%%%%%%%
%%%%%%%%%%%%%%%%%%%%%%%%%%%%%%%%%%%%%%%%%%%%%%%%%%%%%%%%%%%%%%%%%%%%%%%%%%%%%%%

\newpage
\clearpage
\begin{appendices}

\section{Dynamic Batching}
\subsection{Proof of NPC property}
For a directed acyclic graph $G(V,E)$, each node has a type $t\in T$, we define \textit{batch sequence} as a sequence of types $s \in T^*$, that can be used iteratively as the next type in Alg.\ref{alg:Dynamic Batching} to batch the whole dataflow graph. The \textit{Batching} problem is to find a batch sequence with the smallest possible length, denoted as \textit{optimal batching sequence}. 

\begin{theorem}[NP-hard for Batching] \textit{Batching} is NP-hard.
\end{theorem}\label{the: np-hardness for batching}
\begin{proof} 
We prove the NP-hardness by reducing from \textit{Shortest Common Supersequence} (SCS). Given an alphabet $A$, a set of strings, $s_1,s_2,...,s_n$ in $A$, the SCS problem finds the shortest common supersequence for these strings. Treating each letter in the string as a node, the string is a chain of nodes, which is a DAG. Therefore, $s_1,s_2,...,s_n$ compose a DAG with many independent chains. Suppose the \textit{optimal batching sequence} for this DAG is found, we claim that it is exactly the common supersequence for these strings. On one side, every string must appear as a substring in the \textit{optimal batching sequence} to complete the batching. On the other side, if there is a common supersequence shorter than the \textit{optimal batching sequence}, this common supersequence is also a legal batching sequence. This is because that in  the Alg.\ref{alg:Dynamic Batching}, we greedily batch nodes in the frontier once their type is equal to the one in the batching sequence. So it is sufficient for a string to appear as a subsequence to be fully batched. This yields the contradiction. So the \textit{optimal batching sequence} is the common supersequence, indicating SCS can be solved by \textit{Batching} with poly time encoding. So, \textit{Batching} is NP-hard.
\end{proof}

Till now, there is no constant guaranteed approximation algorithm for the SCS problem, and so does \textit{Batching}. 

\subsection {Proof of the sufficient condition on batching} \label{reward proof}
\begin{lemma}
If $\frac{|Frontier(G^{a_t}_t)|}{|Frontier_{a_t}(G_t)|} = 1$, then there exists a shortest batching sequence starting with $a_t$. 
\end{lemma}
\begin{proof}
Proof by contradiction. If this is not true, let $S$ be the set of operations of the first batch whose operation type is $a_t$. Then, we must have $S \subset Frontier(G^{a_t}_t)$. Because $\frac{|Frontier(G^{a_t}_t)|}{|Frontier_{a_t}(G_t)|} = 1$, meaning that $S$ is ready to execute for the first batch. Thus, by moving $S$ to the first batch committed, we get one of the shortest batching sequences starting with $a_t$. Contradiction.
\end{proof}

\subsection{Lower Bound}\label{lower bound}
For a dataflow graph $G$ and type set $T$, The lower bound of kernel launches is given by 
\begin{equation}\label{lowerbound}
|Batching^*(G)| \ge \Sigma_{t\in T} Depth(G_t) 
\end{equation}.
The heuristic behind the formula is that it requires at least $Depth(G_t)$ steps to fully execute the $G_t$. And because of the dependency between $G_t$s, the execution requires at least $\Sigma_{t\in T} Depth(G_t)$ steps to finish.

\subsection{Case the FSM does not cover}

\begin{figure}
\includegraphics[]{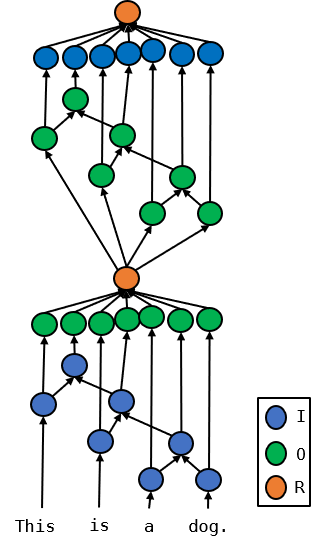}
\caption{Example for when the FSM doesn't work.}
\label{apdx: case not work}
\end{figure}

There are cases when the FSM cannot find a good policy. In the fake example in Fig.\ref{apdx: case not work}, we concatenate two tree networks, but the second has the type of Internal node and Output Node swapped. Here, the FSM in Fig.\ref{fig:state machine example} does not work the first tree requires batching Input node $S_2$ while the second requires batching the Output node. This problem can be solved by introducing the phase information like the portion of nodes committed into the state encoding.

\section{PQ-tree}\label{ref:pqtree}
In this section, we illustrate the functions uncovered in Alg.\ref{alg: PQ tree Memory Arrangement} and give the proof on its time complexity.

\begin{algorithm} [h]
\scriptsize
\caption{PQ tree Memory Allocation}
\label {alg: old pqtree}
\begin{algorithmic}[1]
    \Function{BroadcastConstraint}{$tree$, $\mathcal{B}$}\label{alg:pqtree:old:pass1}
        \State visited = getSet()
        \For {$batch$ in $\mathcal{B}$}
            \If {batch in visited}
                \State continue
            \EndIf 
            \State Q = Queue()
            \State Q.push($batch$)
            \While {not Q.isEmpty()} \label{alg:pqtree:old:iteration}
                \State $b$ = Q.pop();
                \State visited.insert($b$)
                \State $cons = $ \Call{ParseConstraints}{$b$}\label{alg:pqtree:old:parseconstraint}
                \State $suc,updatedBatches = $\Call{ApplyConstraints}{$cons$,$tree$}\label{alg:pqtree:old:applyconstraints}
                \If {$suc$ is False}    
                    \State $\mathcal{B}$.erase(b)
                \Else
                    \For {$b$ in $updatedBatches$}
                        \State Q.push($b$)
                    \EndFor 
                \EndIf
            \EndWhile
        \EndFor
    \EndFunction

    \Function{DecideNodesOrder}{$tree$, $\mathcal{B}$}\label{alg:pqtree:old:pass2}
        \State $POrder$ = getUnionFindSet(tree.PNodes) \Comment{A union-find set to decide QNode's direction.}
        \State $QOrder$ = getUnionFindSet(tree.QNodes) \Comment{A union-find set to decide PNode's permutation.}
        \For {$batch$ in $\mathcal{B}$}
            \State $EquivPairs$ = \Call{ParseEquivNodeOrderPair}{$tree$, $batch$}\label{alg:pqtree:old:parse equi-rel}
            \For {$EquivPair$ in $EquivPairs$}\label{alg:pqtree:old:spread equi-rel}
                \If {$EquivPair$ is a P node pair}
                    \State POrder.Union($EquivPair$)
                \ElsIf {$EquivPair$ is a Q node pair}
                    \State QOrder.Union($EquivPair$)
                \EndIf 
            \EndFor
        \EndFor
        \State \textbf{return} $QOrder$, $POrder$
    \EndFunction 
    
    \Function{Main}{$X$, $\mathcal{B} = [batch_1, ..., batch_n]$} 
        \Comment{$X$ the variable set, $\mathcal{B}$ the batches }
        \State $tree =$ \Call{ConstructPQTree}{$X$, $\mathcal{B}$}
        \State \Call{BroadcastConstraint}{$tree$, $\mathcal{B}$}
        \State $QOrder$, $POrder$ = \Call{DecideNodesOrder}{$tree$, $\mathcal{B}$}
        \State \textbf{return} \Call{GetLeafOrder}{$tree$, $QOrder$, $POrder$}
    \EndFunction
\end{algorithmic}
\end {algorithm}

\textbf{Detailed Illustration}
The supporting functions for \textproc{BroadcastConstraint} are shown in Alg.\ref{alg:pass1 funcs}. The \textproc{FindRoot} function is supported by the \textproc{Bubble} method in the vanilla PQ tree algorithm to search the root for the minimal subtree for a set of leaf roots. The \textproc{ReduceAndGetChanged} method is supported as an extension to the \textproc{Reduce} method in the vanilla PQ tree algorithm to add a consecutive constraint to the PQ tree and record the batches whose tree structure gets changed. It needs to maintain a mapping between the P/Q node with the batches and updates it when the tree structure gets updated in the \textproc{Reduce} step.

\begin{algorithm} [h]
\scriptsize
\caption{Algorithms for functions in \textproc{BroadcastConstraint}}
\label{alg:pass1 funcs}
\begin{algorithmic}[1]
    \Function{getSubtreeCons}{$o$}
        \State $root =$ \Call{FindRoot}{$o$}
        \State $nodeToLeaves = $ \Call{getNodeToLeaves}{root} \Comment{A function maps nodes to leaves in its subtree. Realized by a traversal of the tree on the recursive function $nodeToLeaves(node) = node.isLeaf? \{node\}: \{nodeToLeaves[child] | child \in node.children\} $}
        \State $constraints = getList()$
        \For{$node$ in \Call{getNodesInSubTree}{root}}
            \If {$node$ is P-node}
                \State $cons = \cup_{child\in node.children} nodeToLeaves(child)$
                \State $constaints.push(cons)$
            \ElsIf {$node$ is Q-node} 
                \For {$child \in node.children$}
                    \State $sib = child.nextSibling()$
                    \State $cons = \cup \{nodeToLeaves(child), nodeToLeaves(sib)\}$
                    \State $constraints.push(cons)$
                \EndFor
            \EndIf 
        \EndFor 
        \State \Return $constraints$
    \EndFunction 

    \Function {ParseConstraints}{$constraints$, $batch$}
        \State $uniformConstraints = \cup_{o\in batch.operands}{\{o.index(x)|x\in getSubtreeCons(o)\}}$\Comment{Parse operand-wise consecutive constraint.}
        \State $constraints = getList()$
        \For {$o$ in $batch.operands$} \Comment{Transform constraint by alignment information.}
            \For {$cons$ in $uniformConstraint$}
                \State $constraints.append(\{o[x] | x \in cons\})$
            \EndFor
        \EndFor 
        \State \Return constraints
    \EndFunction 

    \Function{ApplyConstraints}{$constraints$, $tree$,$updatedOperands$}
        \For {$cons$ in $constraints$} 
            \State $suc  =$ \Call{ReduceAndGetChanged}{$tree$, $cons$,$updatedOperands$}
            \If {$suc$ is False} 
                \State \Return False
            \EndIf
        \EndFor
        \State \Return True
    \EndFunction 
\end{algorithmic}
\end {algorithm}

The \textproc{ParseEquivNodeOrderPair} method is given in Fig.\ref{alg: ParseEquivNodeOrderPair} to parse the equivalent node order pairs on the isomorphic tree structures for operands in a batch. It is performed by simultaneous bottom-up traversal for operands in this batch. 

\begin{algorithm} [h]
\scriptsize
\caption{ParseEquivNodeOrderPair}
\label {alg: ParseEquivNodeOrderPair}
\begin{algorithmic}[1]
    \Function{ParseEquivNodeOrderPair}{$tree$, $batch$}
        \State Q = getQueue() \Comment{Queue on equivalent nodes.}
        \For {$i$ in $batch.operands.front().size()$}
            \State Q.push(($o[i] | o \in batch.operands$))
        \EndFor 
        \State EquivNodeOrderPairs = getList()  
        \State Find the root and calculate the leaf count for the subtree of the first operand. 
        \While {True} \Comment{A leaf-to-root search performed parallel on operands in one batch.}
            \State $nodes = $ Q.pop()
            \State $node = nodes.front()$
            \If {$node$ is P node}
                \State $EquivClass = \{(node, node.referenceRrder)|node\in nodes\}$
                \State $EquivNodeOrderPairs.add((P, EquivClass))$
            \ElsIf{$node$ is Q node}
                \State $EquivClass = \{(node, getDirection(node, node.referenceOrder))|node\in nodes\}$
                \State $EquivNodeOrderPairs.add((Q, EquivClass))$
            \EndIf
            \State $node.parent.leafCnt = node.parent.leafCnt - node.leafCnt$ 
            \If {$node.parent.leafCnt $ is 0}
                \State Q.push(($node.parant|node\in nodes$))
            \EndIf 
            \For {node in nodes} \Comment{Reference Order is used to decide the node order.}
                \State node.parent.referenceOrder.append(node)
            \EndFor
            \If $node.isRoot$ \Comment{Stop Condition: Root Found.}
                \State Break.
            \EndIf 
        \EndWhile
        \State \Return EquivNodeOrderPairs
    \EndFunction 
\end{algorithmic}
\end {algorithm}

The methods concerning the Union Find data structure are listed in Alg.\ref{alg: Union-Find set}. In this problem, the UnionFindSet data structure is a set of nodes, and each node has two attributes:1. $parent$, the pointer to the node's parent, or the decider of its order;  2. $\sigma$, the transformation that transforms the node's order (a permutation for P-node or reverse for Q-node) to its parent's. Given a node, the \textproc{Find} method returns the root node of this node and the node's relative order with the root. In \textproc{Find} method, the equivalence relationship between two node order pairs, i.e. ($node_1$, $\sigma_1$), ($node_2$, $\sigma_2$), is built. The constraint conveyed is that if $node_1$ has order $\sigma$ then $node_2$ must have order $\sigma \circ \sigma_1^{-1}\sigma{2}$, and this is encoded into the data structure by building a relationship between their roots. If their roots are not the same, an edge connects them with the transformation satisfying the information (line \ref{alg:union: assign order}). If they are the same, then $node_1,node_2$'s relative order to the root must satisfy the constraint (line \ref{alg:union: compatible check}). Otherwise, the equivalence relationship is not compatible and this relationship is dropped.

\begin{algorithm} [h]
\scriptsize
\caption{Extended Union-Find set algs}
\label {alg: Union-Find set}
\begin{algorithmic}[1]
    
    \Function{getUnionFindSet}{$nodes$}
        \For {node in nodes}
            \State node.parent = node
            \State node.$\sigma$ = I \Comment{Identical transformation}
        \EndFor 
    \EndFunction
    
    \Function {Find} {$node$}
        \State $\sigma = I$ \Comment{The order relative to the root.}
        \While {$node.parent$ is not $node$}
            \State $\sigma = \sigma \circ node.\sigma$
            \State $node = node.parent$
        \EndWhile
        \State \Return node, $\sigma$
    \EndFunction 

    \Function{Union} {$node_1$, $\sigma_1$, $node_2$, $\sigma_2$} 
        \State $p_1$, $\sigma_3$ = \Call{Find}{$node_1$} \label{alg:union:findparent1}
        \State $p_2$, $\sigma_4$ = \Call{Find}{$node_2$}\label{alg:union:findparent2}
        \If {$p_1$ is not $p_2$} 
            \State $p_1.parent = p_2$
            \State $p_1.\sigma = \sigma_3^{-1}\sigma_4\sigma_2^{-1}\sigma_1$ \label{alg:union: assign order}
        \ElsIf {$\sigma_1^{-1}\sigma_2$ is $\sigma_3^{-1}\sigma_4$}\label{alg:union: compatible check} \Comment{Compatibale}
            \State Do nothing. Already equivalent.
        \Else \Comment{Incompatible.}
            \State \Return False 
        \EndIf 
        \State \Return True
    \EndFunction 
\end{algorithmic}
\end {algorithm}

Finally, we obtain the memory allocation sequence by a depth-first traversal satisfying the constraint we found on the node order (Alg.\ref{alg: GetLeaf}). 

\begin{algorithm} [h]
\scriptsize
\caption{Get Memory Allocation Order}
\label {alg: GetLeaf}
\begin{algorithmic}[1]
\Function{GetLeafOrder}{$tree$, $POrder$, $QOrder$}
    \State root = tree.root
    \State order = getList()
    \State S = getStack()
    \State S.push(root)
    \While {S.notEmpty()} \Comment{Depth first traversal}
        \State node = S.pop()
        \If {node is P node}
            \State p, $\sigma$ = POrder.find(node)
            \For {child in $\sigma(node.children)$}
                \State \Call{GetLeafOrder}{child}
            \EndFor 
        \ElsIf {node is Q node}
            \State p, $direction$ = POrder.find(node)
            \For {child in node.children following $direction$}
                \State \Call{GetLeafOrder}{child}
            \EndFor             
        \Else \Comment{Leaf Node hear}
            \State order.append(order)
        \EndIf 
    \EndWhile 
    \State \Return order
\EndFunction
\end{algorithmic}
\end {algorithm}

\textbf{Complexity}\label{pqtree complexity}

\begin{lemma}
For the batching problem, PQ tree memory allocation algorithm's time complexity is $O(\Sigma_{batch\in batches}|batch|\max^2_{batch\in batches}|batch|)$ where $|.|$ counts the operation in a batch.
\end{lemma}
\begin{proof}
The \textproc{Reduce} step on a consecutive constraint $S$ in the PQ tree is $O(|S|)$. Thus, time complexity for PQ-tree construction is $O(\Sigma_{batch\in \mathcal{B}} |batch|)$. 
In the \textproc{BroadcastConstraint} pass, the while-loop body (line \ref{alg:pqtree:old:iteration}) can only perform O($\Sigma_{batch\in \mathcal{B}} |batch|$). This is because every update on the PQ-tree structure either transfers a P-node into a Q-node or introduces a new node, the total times for updates on the tree structure are bounded by the number of internal nodes for the PQ-tree and are further bounded by the number of leaf variables.
Then, for the while-loop body, the \textproc{getSubTreeCons} method on an operand with $k$ variables needs $O(k^2)$ to compute as the \textproc{getNodeToLeaves} method needs to assign each node with the set of its leaves and the number of nodes is bounded by $k$. It is not hard to see the rest of \textproc{ParseConstraint} has lower complexity. Thus, for a batch with $m$ variables, it takes $O(mk^2)$ to compute the \textproc{ParseConstraint}. For the \textproc{ApplyConstraints} step, the \textproc{ReduceAndGetChanged} step can be implemented to have the same complexity as \textproc{Reduce}, where a bi-direction map is used to store the relationship between the node and the batch, once it a node is deleted or inserted, a callback is used to update this table in constant time. Then, \textproc{ApplyConstraints}'s complexity is bounded by the sum of variables in the constraints, which is also $O(mk^2)$. 
In all, for the \textproc{BroadcastConstraint} pass, suppose there are $n$ variables, the time complexity is $O(nmk^2)$. Under the batching setting, each node appears in the result operand of a batch once and only once, $\Sigma_{batch\in batches}|batch| = nm$, and $k \le \max_{batch\in batches}|batch|$. Thus, the time comlexity is bounded by $O(\Sigma_{batch\in batches}|batch|\max^2_{batch\in batches}|batch|)$. 

For the \textproc{DecideNodesOrder} function, \textproc{ParseEquivNodeOrderPair} on $batch$ requires $O(|batch|)$ time to traverse the graph. The $Union$ method is $O(\alpha(n))$, where $\alpha(n)$ is the extremely slow-growing inverse Ackermann function. Thus, the time complexity is $O(\Sigma_{|batch|\in batches}(|batch| + O(\alpha(n)))) \approx O(\Sigma_{|batch|\in batches}(|batch|))$.

In all, the time complexity of the PQ tree memory allocation algorithm is $O(\Sigma_{batch\in batches}|batch|\max^2_{batch\in batches}|batch|)$.

\end{proof}

\section{More On Comparison with Cortex}\label{sec:appendix:cortex}
 Cortex\cite{cortex} is highly specialized for optimizing the recursive neural network and it requires the user to express the tensor computation and specify the optimization through TVM's domain-specific language~\cite{tvm}. This framework is fundamentally different from general frameworks like \name and DyNet in that it doesn't rely on vendor libraries and the user is given a full chance to optimize computation from the graph level to the operation level. This gives expert users the chance to squeeze the performance by specializing the application to the hardware but is burdensome for common users who basically want to prototype an application. 
 
 The experiment for the comparison between \name~with Cortex performs on TreeLSTM and TreeGRU. Cortex doesn't support the LSTM-NMT model because it has a tensor-dependent control flow. We did not compare the rest of the models basically because of the lack of expertise in writing the schedules in TVM. The optimization we used in cortex includes the automated \textit{linearization} and \textit{auto-batching}. For the user-given optimizations, we did not perform kernel fusion and the individual operators were optimized by loop transformations like loop tiling, loop reorder, and axis binding. In the end, for the TreeLSTM case it takes us 30 lines of python code to specify the computations and 105 lines of TVM schedules to optimize the kernel.

\end{appendices}

\end{document}